\documentclass{article}

\input{preamble}

\newcommand{\famreg}{\textsc{FAM}\xspace}

\newtheorem*{lemma*}{Lemma}

\newcommand{\mbf}[1]{\mathbf{#1}}

\newcommand{\loss}{\ell}

\newcommand{\RR}{\mathbb{R}}

\newcommand{\Xcal}{\mathcal{X}}
\newcommand{\Ycal}{\mathcal{Y}}
\newcommand{\Ecal}{\mathcal{E}}

\newcommand{\w}{\mathbf{w}}
\newcommand{\W}{\mathbf{W}}
\newcommand{\inner}[2]{\left\langle #1, #2\right\rangle}
\newcommand{\bigo}[1]{\mathcal{O}\left(#1 \right)}

\newif\ifapx
\apxtrue 

\newcommand{\ourmethod}{\famreg}
\newcommand{\ourmaintitle}{\ourmethod: Relative Flatness Aware Minimization}

\newcommand{\ourtitle}{\ourmaintitle}

\usepackage{xcolor}\definecolor{darkblue}{rgb}{0,0,0.75}
\usepackage[pagebackref, breaklinks,colorlinks=true, citecolor=darkblue, urlcolor=darkblue,linkcolor=darkblue]{hyperref}
\graphicspath{ ./figs/ }

\ifpdf
\usetikzlibrary{external}
\fi
\usetikzlibrary{patterns}

\usepackage{eqparbox}

\usepackage[accepted]{icml2023}

\icmltitlerunning{\ourtitle}

\begin{document}
	
	\twocolumn[
        \icmltitle{\ourtitle}

        \begin{icmlauthorlist}
        \icmlauthor{Linara Adilova}{rub}
        \icmlauthor{Amr Abourayya}{ikim}
        \icmlauthor{Jianning Li}{ikim}
        \icmlauthor{Amin Dada}{ikim}
        \icmlauthor{Henning Petzka}{lund}
        \icmlauthor{Jan Egger}{ikim,graz}
        \icmlauthor{Jens Kleesiek}{ikim}
        \icmlauthor{Michael Kamp}{ikim,rub,monash}
        
        \end{icmlauthorlist}
        
        \icmlaffiliation{ikim}{Institute for AI in medicine (IKIM) at University Hospital Essen, Essen, Germany}
        \icmlaffiliation{rub}{Ruhr-University Bochum, Bochum, Germany}
        \icmlaffiliation{monash}{Monash University, Melbourne, Australia}
        \icmlaffiliation{graz}{Graz University, Graz, Austria}
        \icmlaffiliation{lund}{Lund University, Lund, Sweden}
        
        \icmlcorrespondingauthor{Linara Adilova}{linara.adilova@rub.de}

        \icmlkeywords{flatness, relative flatness, regularization, optmization, generalization, deep learning}
        
        \vskip 0.3in
    ]
    
    \printAffiliationsAndNotice{}  

	\begin{abstract}
		Flatness of the loss curve around a model at hand has been shown to empirically correlate 
with its generalization ability.
Optimizing for flatness has been proposed as early as 1994 by Hochreiter and Schmidthuber, and was followed by more recent successful sharpness-aware optimization techniques. 
Their widespread adoption in practice, though, is dubious because of 
the lack of theoretically grounded connection between flatness and generalization, in particular in light of the reparameterization curse---certain reparameterizations of a neural network change most flatness measures but do not change generalization.
Recent theoretical work suggests that a particular relative flatness measure can be connected to generalization and solves the reparameterization curse. 
In this paper, we derive a regularizer based on this relative flatness that is easy to compute, fast, efficient, and works with arbitrary loss functions.
It requires computing the Hessian only of a single layer of the network, which makes it applicable to large neural networks, and with it avoids an expensive mapping of the loss surface in the vicinity of the model. 
In an extensive empirical evaluation we show that this relative flatness aware minimization (FAM) improves generalization in a multitude of applications and models, both in finetuning and standard training. 
We make the code available at \href{https://github.com/kampmichael/RelativeFlatnessAndGeneralization/tree/main/RelativeFlatnessRegularizer(FAM)}{github}.


	\end{abstract}
	
	\section{Introduction}
It has been repeatedly observed that the generalization performance of a model at hand correlates with flatness of the loss curve, i.e., how much the loss changes under perturbations of the model parameters~\citep{entropySGD, keskarLarge, foret2021sharpnessaware, zheng2020regularizing, sun2020exploring, adversarial_weightPerturbation, fisherRao, yao2018hessian}. The large-scale study by \citet{jiang2020fantastic} finds that such flatness-based measures have a higher correlation with generalization than alternatives like weight norms, margin-, and optimization-based measures. The general conclusion is that flatness-based measures show the most consistent correlation with generalization.

Naturally, optimizing for flatness promises to obtain better generalizing models. \citet{hochreiter1994simplifying} proposed a theoretically solid approach to search for large flat regions by maximizing a box around the model in which the loss is low. More recently, it was shown that optimizing a flatness-based objective together with an L2-regularization performs remarkably well in practice on a variety of datasets and models~\citep{foret2021sharpnessaware}. 
The theoretical connection to generalization has been questionable, though, in particular in light of negative results on reparametrizations of ReLU neural networks~\citep{dinhSharp}: these reparameterizations change traditional measures of flatness, yet leave the model function and its generalization unchanged, making these measures unreliable.

Recent work~\citep{petzka2021relative} has shown that generalization can be rigorously connected to flatness of the loss curve, resulting in a relative flatness measure that solves the reparameterization issue. That is, the generalization gap of a model $f:\Xcal\rightarrow\Ycal$ depends on properties of the training set and a measure
\[
\kappa(\w^l) := \sum_{s,s'=1}^d \langle \w^l_s,\w^l_{s'}\rangle  \cdot Tr(H_{s,s'}(\w^l))\enspace ,
\]
where $\w^l\in\RR^{d\times m}$ are the weights between a selected layer $l$ with $m$ neurons and layer $l+1$ with $d$ neurons. Further, $\langle \w^l_s,\w^l_{s'}\rangle=\w^l_s(\w_{s'}^l)^T$ is the scalar product of two row vectors (composed of the weights into neurons with index $s$ and $s'$ in layer $l+1$), $Tr$ denotes the trace, and $H_{s,s'}$ is the Hessian matrix containing all partial second derivatives with respect to weights in rows $\w^l_s$ and $\w^l_{s'}$: 
\begin{equation*}
    H_{s,s'}(\w,f(S))=\left [\frac{\partial^2 \mathcal{E}_{emp}(f,S)}{\partial w_{s,t} \partial w_{s',t'}} \right ]_{1\leq t,t'\leq m}\enspace .
\end{equation*}
Here, $\mathcal{E}_{emp}$ is the empirical risk
\[
\Ecal_{emp}(f,S)=\frac 1{n} \sum_{i=1}^{n} \loss(f(x_i),y_i)
\]
on a dataset
\[
S=\{(x_1,y_1),\dots,(x_n,y_n)\}\subset \Xcal\times\Ycal\enspace .
\]
It is demonstrated that, measured on the penultimate layer, this measure highly correlates with generalization.
Sharpness-aware minimization (SAM)~\citep{foret2021sharpnessaware} also optimizes for a measure of flatness, but is not reparameterization invariant---even under L2-regularization its invariance is unclear, in particular wrt. neuronwise reparameterizations.
The reparamterization-invariant extension of SAM, ASAM~\citep{kwon2021asam} is not theoretically connected to generalization.

In this paper, we implement the \textit{relative flatness} measure of \citet{petzka2021relative} as a regularizer for arbitrary loss functions and derive its gradient for optimization. A remarkable feature of the relative flatness measure is that it is only applied to a single layer of a neural network, in comparison to classical flatness (and sharpness) which takes into account the entire network. \citet{petzka2021relative} have shown that relative flatness in this layer corresponds to robustness to noise on the representation produced by this layer. Therefore, FAM nudges the entire network to produce a robust representation in the chosen layer. 
At the same time, it does not require flatness wrt.\ the other weights, opening up the design space for good minima. Since it suffices to compute relative flatness wrt.\ a single layer, this regularizer and its gradient can be computed much more efficiently than any full-Hessian based flatness measure. Moreover, since the gradient can be computed directly, no double backpropagation is required.

In an extensive empirical evaluation we show that the resulting relative flatness aware minimization (FAM) improves the generalization performance of neural networks in a wide range of applications and network architectures: We improve test accuracy on image classification tasks (CIFAR10, CIFAR100, SVHN, and FashionMNIST) on ResNET18 (outperforming reported best results for this architecture), WideResNET28-10, and EffNet-B7 and compare it to SAM regularizer. In a second group of experiments we reduce DICE-loss substantially on a medical shape reconstruction tasks using autoencoders and stabilize the language model finetuning.

\pagebreak
Our contributions are 

(i) a novel regularizer (FAM) based on relative flatness that is easy to implement, flexible, and compatible with any thrice-differentiable loss function, and 

(ii) an extensive empirical evaluation where we show that FAM regularization improves the generalization performance of a wide range of neural networks in several applications.

	\section{Related Work}

\subsubsection{Flatness as generalization measure} 
Flatness of the loss surface around the weight parameters is intimately connected to the amount of information that the model with these parameters can be described with, i.e., if the region is flat enough and loss does not change, the parameters can be described with less precision still allowing to have a good performing model.
Correspondingly, the models in the flat region generalize better: \citet{hochreiter1994simplifying} investigated a regularization that leads to a flatter region in the aforementioned sense.
Their results have shown that indeed such optimization leads to better performing models.
Following up, flatness of a minimizer was used to explain generalization abilities of differently trained neural networks~\citep{keskar2016large}, where it was specifically emphasized that calculation of a Hessian for modern models is prohibitively costly.
Originating from the minimum description length criteria for finding better generalizing learning models, flatness became a pronounced concept in the search for generalization criteria of large neural networks.
The PAC-Bayes generalization bound rediscovers the connection of the Hessian as flatness characteristic with the generalization gap and the large-scale empirical evaluation~\citep{jiang2020fantastic} shows that all the generalization measures based on flatness (in some definition) highly correlate with the actual performance of models.

\citet{petzka2021relative} considered an analytical approach to connect flatness with generalization, resulting in the measure of \emph{relative flatness} that is used for regularization in this work. 
Their approach splits the generalization gap into two parts termed feature robustness and representativeness.
While representativeness measures how representative the training data is for its underlying distribution, feature robustness measures the loss at small perturbations in feature layers.
Under the assumption of representative data, so that the training data can indeed be used to learn something about the true underlying distribution, feature robustness governs the generalization gap. 
Further, at a local minimum, feature robustness is exactly described by relative flatness if target labels are locally constant (i.e., labels do not change under small perturbations of features).
This condition of locally constant labels is identified as a necessary condition for connecting flatness to generalization.
Finally, the paper demonstrates how the measure of relative flatness solves the reparameterization-curse discussed in \citet{dinh2017sharp}, rendering itself as a good candidate for an impactful regularizer.

\subsubsection{Regularizing optimization} 
Regularization (implicit or explicit) is de facto considered to be an answer to the good generalization abilities of an overparametrized model.
New elaborate techniques of regularization allow to beat state-of-the-art results in various areas.
Obviously, flatness can be considered as a good candidate for a structural regularization, but since the size of the modern models grew significantly after the work of \citet{hochreiter1994simplifying}, straightforward usage of the initial flatness measures is not feasible in the optimization.
Analogously, approaches to flatness stimulation from averaging over solutions~\citep{izmailov2018averaging} cannot be backpropagated and directly used in the optimization process.
The closest research to the flatness optimization is related to adversarial robustness---adversarial training aims at keeping the loss of a model on a constant (low) level in the surrounding of the training samples, which can be also done in the feature space~\citep{wu2020adversarial}.
Several recent works proposed an optimizer for neural networks that is approximating the minimax problem of minimizing loss in the direction of the largest loss in the surrounding of the model.
One of them, sharpness aware minimization (SAM)~\citep{foret2021sharpnessaware} achieves state-of-the-art results in multiple tasks, e.g., SVHN, and allows for simple backpropagation through the proposed loss.
However, the exact proposed $m$-sharpness does not entirely correspond to the theoretical motivation proposed by \citet{foret2021sharpnessaware} based on PAC-Bayes generalization bound, which might mean that the empirical success of SAM and its variants~\citep{kwon2021asam, zhuangsurrogate, du2021efficient, liu2022towards, liu2022random} cannot be explained by theoretical PAC-Bayes flatness of the solution~\citep{andriushchenko2022towards, wen2022does}.
Thus, introducing a theoretically grounded flatness regularizer can be of an interest for community.


	\section{Flatness Aware Minimization}
\label{sec:algo}

In the following we give a detailed description of the proposed regularization.
For a differentiable loss function $\loss(S,\W)$ and a training set $S$, the regularized objective is
\[
\loss(S,\W) + \lambda \kappa(\w^l)\enspace ,
\]
where $\lambda$ is the regularization coefficient and $\w^l\in\RR^{m\times d}$ denote the weights from layer $l$ to $l+1$. To optimize this objective, we compute its gradient (and omit the training set $S$ in the notation for compactness):
\begin{equation}
\nabla_\W \left(\loss(\W) + \lambda \kappa(\w^l)\right) = \nabla_\W \loss(\W) + \lambda \nabla_\W \kappa(\w^l)
\label{eq:famloss}
\end{equation}
Here, $\nabla_\W \loss(\W)$ is the standard gradient of the loss function. It remains to determine $\nabla_\W  \kappa(\w^l)$.
\begin{lemma} For a neural network with $L$ layers and weights $\W=(\w^1,\dots,\w^L)$ with $\w^k\in\RR^{O^k\times P^k}$ and a specific layer $l\in [L]$ with weights $\w^l\in\RR^{d\times m}$ it holds that
\begin{equation*}
    \begin{split}
        &\nabla_\W \kappa(\w^l) = e^l \left[2\sum_{s=1}^d w^l_s Tr\left(H_{s,i}\right)\right]_{i\in [d]}+\\
        &\left(\left[\sum_{s,s'=1}^d \inner{w^l_s}{w^l_{s'}} \sum_{t=1}^m \frac{\partial^3 \loss(\W)}{\partial\w^k_{o,p}\partial\w^l_{s,t}\partial\w^l_{s',t}}\right]_{\substack{\\p\in [P^k] \\ o \in [O^k]}}\right)_{k\in [L]}
    \end{split}
\end{equation*}
where $e^l$ denotes the $l$-th standard unit vector in $\RR^L$.
\label{lm:derivative}
\end{lemma}
\begin{proof}
For this proof, we simply apply product rule on the gradient of the regularizer, yielding two parts that we separately simplify. 
\begin{equation*}
    \begin{split}
        \nabla_\W & \kappa(\w^l) = \nabla_\W \sum_{s,s'=1}^d \inner{w^l_s}{w^l_{s'}} Tr\left(H_{s,s'}\right)\\
        =&\sum_{s,s'=1}^d \left(\nabla_\W\inner{w^l_s}{w^l_{s'}}\right) Tr\left(H_{s,s'}\right)\\
        &+\sum_{s,s'=1}^d \inner{w^l_s}{w^l_{s'}} \nabla_\W Tr\left(H_{s,s'}\right)\\
        =&\left[\underbrace{ \sum_{s,s'=1}^d \left(\frac{\partial}{\partial \w^k} \inner{w^l_s}{w^l_{s'}}\right) Tr\left(H_{s,s'}\right) }_{(I)}\right]_{1\leq k \leq L}\\       
        &+\left[\underbrace{ \sum_{s,s'=1}^d \inner{w^l_s}{w^l_{s'}} \frac{\partial}{\partial \w^k}Tr\left(H_{s,s'}\right) }_{(II)}\right]_{1\leq k \leq L}
    \end{split}
\end{equation*}
Let us simplify both parts, starting with $(I)$, which is $=0$ for all $k\neq l$. For $k=l$ it simplifies to
\begin{equation*}
    \begin{split}
        & \sum_{s,s'=1}^d \left(\frac{\partial}{\partial \w^l} \inner{w^l_s}{w^l_{s'}}\right) Tr\left(H_{s,s'}\right)\\
        =& \left[\sum_{s,s'=1}^d \left(\frac{\partial}{\partial \w^l_i} \inner{w^l_s}{w^l_{s'}}\right) Tr\left(H_{s,s'}\right)\right]_{1\leq i\leq d}\\
    \end{split}
\end{equation*}
Now for each $i\in[d]$ we have that
\begin{equation*}
    \begin{split}
        &\sum_{s,s'=1}^d \left(\frac{\partial}{\partial \w^l_i} \inner{w^l_s}{w^l_{s'}}\right) Tr\left(H_{s,s'}\right)\\
        =&2\sum_{s=1}^d w^l_s Tr\left(H_{s,i}\right)\enspace ,
    \end{split}
\end{equation*}
where we have used the symmetry of $H_{s,s'}$ and the commutativity of the inner product in the last step. 
Therefore, it holds that
\begin{equation*}
    \begin{split}
& \sum_{s,s'=1}^d \left(\frac{\partial}{\partial \w^l} \inner{w^l_s}{w^l_{s'}}\right) Tr\left(H_{s,s'}\right)\\
=& \left[2\sum_{s=1}^d \inner{w^l_s}{w^l_i} Tr\left(H_{s,i}\right)\right]_{1\leq i\leq d}\enspace .
    \end{split}
\end{equation*}

For the second part $(II)$, let $\w^k\in \RR^{O\times P}$. Then, $\frac{\partial}{\partial \w^k}Tr\left(H_{s,s'}\right)$ can be expressed as
\begin{equation*}
    \begin{split}
        \frac{\partial}{\partial \w^k}Tr\left(H_{s,s'}\right)=&\frac{\partial}{\partial \w^k}Tr\left[\frac{\partial^2 \loss(\W)}{\partial\w^l_{s,t}\partial\w^l_{s',t'}}\right]_{1\leq t,t'\leq m}\\
        =&\frac{\partial}{\partial \w^k}\sum_{t=1}^m \frac{\partial^2 \loss(\W)}{\partial\w^l_{s,t}\partial\w^l_{s',t}}\\
        =&\left[\sum_{t=1}^m \frac{\partial^3 \loss(\W)}{\partial\w^k_{o,p}\partial\w^l_{s,t}\partial\w^l_{s',t}}\right]_{\substack{1\leq p\leq P \\ 1\leq o \leq O}}
    \end{split}
\end{equation*}
Putting $(I)$ and $(II)$ together finally yields
\begin{equation*}
    \begin{split}
        \nabla_\W & \kappa(\w^l) = e^l \left[2\sum_{s=1}^d \inner{w^l_s}{w^l_i} Tr\left(H_{s,i}\right)\right]_{1\leq i\leq d}\\
        &+\left[\sum_{s,s'=1}^d \inner{w^l_s}{w^l_{s'}} \sum_{t=1}^m \frac{\partial^3 \loss(\W)}{\partial\w^k_{o,p}\partial\w^l_{s,t}\partial\w^l_{s',t}}\right]_{\substack{1\leq k \leq L\\1\leq p\leq P^k \\ 1\leq o \leq O^k}}
    \end{split}
\end{equation*}
where $e^l$ denotes the $l$-th standard unit vector in $\RR^L$.
\end{proof}

\subsection{Computational Complexity}
Computing the FAM regularizer requires computing the Hessian wrt. the weights $\w^l\in\RR^{d\times m}$ of the feature layer, which has computational complexity in $\bigo{d^2m^2}$. From this, the individual $H_{s,s'}$ can be selected. The inner product computation has complexity $\bigo{dm}$, so that the overall complexity of computing the regularizer is in $\bigo{d^2m^2}$.

In order to train with the FAM regularizer, we have to compute the gradient of the regularized loss wrt. the weights $\W$ of the network. Computing the gradient of the loss function in equation~\ref{eq:famloss} has complexity $\bigo{|\W|}$, where $|\W|$ denotes the number of parameters in $\W$. The computation of $\nabla_\W \kappa(\w^l)$ is decomposed into the sum of two parts in Lemma~\ref{lm:derivative}. The first part has complexity $\bigo{d^2m^2}$ for computing the Hessian and the inner product, as before. All parts in the sum, however, have already been computed when computing $\kappa(\w^l)$. The second part requires computing the derivative of the Hessians $H_{s,s'}$ wrt. each parameter in $\W$. Since we only need to compute the derivative wrt. the trace, i.e., the sum of diagonal elements, the complexity is in $\bigo{\W}$. Therefore, the overall complexity of computing the FAM regularizer is in 
\[
\bigo{|\W| + d^2m^2 +|\W|}=\bigo{|\W| + d^2m^2}\enspace . 
\]
That is, the additional computational costs for using the FAM regularizer is in $\bigo{d^2m^2}$ per iteration, i.e., in the squared number of weights of the selected feature layer.

In practice, computational time currently can exceed the vanila and SAM training time on $20\%-40\%$.
We also observe that GPU utilization for our implementation is still not optimal, which possibly can lead to delays in computation.

\subsection{A Simplified Relative Flatness Measure}
\label{subsec:simpl-measure}
A more computationally efficient approximation to relative flatness, proposed by \citet{petzka2019reparameterization}, does not iterate over individual neurons, but computes the weight norm of layer $l$ and the trace of the Hessian wrt. layer $l$:
\[
\widehat{\kappa}(\w^l)=\|\w^l\|^2_2 Tr\left(H \right)\enspace .
\]
Computing this measure not only avoids the loop over all pairs of neurons $s,s'\in [d]$, but also allows us to approximate the trace of the Hessian, e.g., with Hutchinson's method~\citep{yao2020pyhessian}. 
On top of the computational efficiency, the trace approximation reduces the memory footprint, enabling us to employ FAM regularization to even larger layers---including large convolutional layers. 

We provide details on the implementation of Hessian computation and Hessian trace approximation in Appendix~\ref{sec:hesscomp}.

	\section{Experiments}
\label{sec:experiments}
In the following section we describe the empirical evaluation of the proposed flatness regularization.
We compare the performance of FAM to the baseline without flatness related optimization and to SAM.
We use the SAM implementation for pytorch~\footnote{\url{https://github.com/davda54/sam}} with the parameters of the base optimizer recommended by the authors.
It should be mentioned here that no matter of its popularity there is no official pytorch implementation of the SAM optimizer, which results in multitude of different implementations for each of the paper using the approach.
Moreover, there are multiple tricks that should be considered when using SAM, e.g., one should take care of normalization layers and check on which of the two optimization steps they are active or non-active.
We run SAM for the same amount of epochs that FAM and simple optimization, no matter that in the original work the authors doubled the amount of epochs for non-SAM approaches due to the doubled run time, thus giving SAM an advantage in our experiments.
Reported result for one of the pytorch implementations of SAM on CIFAR10 with ResNet20 is $93.5\%$ test accuracy~\footnote{\url{https://github.com/moskomule/sam.pytorch}}.
This is the closest reported result to our setup and it should be expected that ResNet18 shows worse result than ResNet20.
Unfortunately, the results for CIFAR100, SVHN, and FashionMNIST are not reported in the implementations of SAM for pytorch. 

We use the FAM regularizer computed on the penultimate layer (or bottleneck layer), since it was demonstrated to be predictive of generalization in \citet{petzka2021relative}. Investigating the impact of the regularizer on other layers is left for future work.

\paragraph{Note on other flat-minima optimizers:\newline}
There are several extensions of SAM~\citep{kwon2021asam, zhuangsurrogate, du2021efficient, liu2022towards, liu2022random} and other flat-minima optimizers, e.g., \citep{chaudhari2019entropy, sankar2021deeper}.
We follow~\citet{kaddourflat} and do not consider them in this work due to their computational cost and/or lack of performance gains compared to original SAM.

\subsection{Image Classification}

Standard datasets for image classification are the baseline experiments that confirm the effectiveness of the proposed regularization.
In particular, we worked with CIFAR10 and CIFAR100~\citep{krizhevsky2009learning}, SVHN~\citep{netzer2011reading}, and FashionMNIST~\citep{xiao2017/online}.
We compare our flatness regularized training to the state-of-the-art flatness regularizer SAM.
For this group of experiments we used the setups from the original SAM paper in order to compare to its performance.
Nevertheless, due to the different implementation, the exact numbers reported seem to be unachievable---while we still see the improvement from using SAM optimizer, both no regularization baseline and SAM baseline are lower than in the original paper.
For all experiments in this group we use the original neuronwise flatness measure for regularization without approximations introduced in Sec.~\ref{subsec:simpl-measure}.

\subsubsection{CIFAR10}
We have chosen ResNet18 as an architecture to solve CIFAR10.
While ResNet18 is not the state of the art for this problem, it allows to confirm the hypothesis about performance of our method.
The reported accuracy of this architecture on CIFAR10 is $95.55\%$.
In our experiments we compare this baseline, that is not using flatness-related optimizations to SAM approach and our proposed regularization.
Standard augmentation strategy is applied, including randomized cropping and horizontal flipping and normalization of the images.
For baseline training we use the following parameters of optimization: SGD with batch size $64$, weight decay of $5e-4$, momentum $0.9$, and cosine annealing learning rate starting at $0.03$ during $250$ epochs.
For FAM the optimizer parameters are kept same and $\lambda$ selected to be $0.1$.
Finally SAM was ran with SGD with a scheduler learning rate $0.01$ and momentum $0.9$.

We report the results achieved in Table~\ref{table:image-classification} in the row corresponding to CIFAR10.

\subsubsection{CIFAR100}
For solving this dataset we follow the approach taken by \citet{foret2021sharpnessaware}. 
We use an EfficientNet~\citep{tan2019efficientnet} (EffNet-B7) that is pretrained on ImageNet and then finetune it for CIFAR100.
In order to obtain a good finetuning result the inputs should be at least in the ImageNet format ($224 \times 224$) which significantly slows down the training.
For standard training and FAM regularized training, the Adam optimizer had consistently the highest performance (compared to SGD and rmsprop) with a batch size of $B=32$. The architecture achieves a baseline accuracy of $84.6$ without regularization, and SAM achieves an accuracy of $85.8$. The FAM regularizer improves the accuracy to $87.15$.
Note, that different from other setups, where both SAM and FAM were used together with the same hyperparameters as the vanila training, here we optimized parameters separately to avoid overfitting.

We report the results achieved in Table~\ref{table:image-classification} in the row corresponding to CIFAR100.

\subsubsection{SVHN and FashionMNIST}
Both SVHN and FashionMNIST problems are reported to reach state-of-the-art performance with SAM optimization using WideResNet28-10 architecture~\citep{zagoruyko2016wide}.
It should be noted that SAM achieves the reported state-of-the-art result on these datasets when combined together with shake-shake regularization technique~\citep{gastaldi2017shake}, which we omitted.

The results reported by \citet{foret2021sharpnessaware} for SVHN are obtained using the training dataset that includes extra data (overall $\sim 600000$ images).
Due to the time constraints we report results of training using only main training dataset ($\sim 70000$ images).
We apply AutoAugment SVHN policy~\citep{cubuk2018autoaugment}, random cropping and horisontal flip, and cutout~\citep{devries2017cutout} with $1$ hole of length $16$.
Our training parameters are $100$ epochs, learning rate of $0.1$ with a multistep decay by $0.2$ after $0.3$, $0.6$ and $0.8$ of the training epochs, batch size of $128$, optimizer is Nesterov SGD with momentum of $0.9$ and weight decay of $5e-4$.
For FAM we use $\lambda=0.1$.

We modify FashionMNIST to have three channels, resize to $32 \times 32$, apply cutout with $1$ hole of length $16$, and normalize by $0.5$.
The training of FashionMNIST is very unstable and has oscillating learning curves with and without regularization.
The used batch size is $64$, learning rate is $0.01$ with the same learning rate scheduler as for SVHN, the training is done for $200$ epochs. Weight decay and momentum are set as in SVHN training.

Finally, in order to apply more computationally expensive neuronwise flatness regularization, we add one more penultimate fully-connected layer in the architecture of WideResNet with $64$ neurons.
Our experiments reveal that this additional layer does not change the outcome of the training in case of non-flatness regularized run.

With the described setup we did not achieve the accuracy reported in the original paper, that are $0.99 \pm 0.01$ error for SAM on SVHN with auto-augmentation and $1.14 \pm 0.04$ for baseline training on SVHN with auto-augmentation; $3.61 \pm 0.06$ error for SAM on FashionMNIST with cutout and $3.86 \pm 0.14$ for baseline training on FashionMNIST with cutout. We report the results achieved in Table~\ref{table:image-classification} in the rows corresponding to SVHN and FashionMNIST.

\begin{table*}[ht]
\centering
\caption{Results for Image Classification Tasks}
\begin{tabular}[t]{lccc} 
\toprule
 &  Baseline &  SAM & FAM \\
\hline
 \textbf{CIFAR10}      & $95.53 \pm 0.0001$ & $95.61 \pm 0.001$          & $\mathbf{95.62 \pm 0.002 }$  \\
\textbf{CIFAR100}      & $84.48 \pm 0.12$     & $85.72 \pm 0.08$            & $\mathbf{87.2 \pm 0.05}$  \\
 \textbf{SVHN}         & $97.72 \pm 0.02$   & $\mathbf{97.84 \pm 0.05}$  & $97.81 \pm 0.07$  \\
 \textbf{FashionMNIST} & $94.57 \pm 0.28$   & $\mathbf{94.99 \pm 0.02}$  & $94.6 \pm 0.04$  \\
\bottomrule
\end{tabular}
\label{table:image-classification}
\end{table*}


\subsection{Medical Shape Reconstruction}
3D shape reconstruction has important applications in both computer vision \citep{smith20203d,chibane2020implicit} and medical imaging \citep{amiranashvili2022learning,li2021autoimplant}. 
Machine learning methods for shape reconstruction have become increasingly popular in recent years, however, often suffer from bad generalization, i.e., a neural network cannot generalize properly to shape variations that are not seen during training. 
In this experiment, we demonstrate that FAM regularizer can effectively mitigate the generalization problem in a skull shape reconstruction task, where a neural network learns to reconstruct anatomically plausible skulls from defective ones \citep{li2021autoimplant,kodym2020skull}.
Here, due to the large size of the layers, we used the approximated layerwise flatness measure for FAM optimization.
We do not include SAM baseline here, while the general question here is whether flatness can aid in generalization for the current task.

\subsubsection{Dataset}
The skull dataset used in this experiment contains $100$ binary skull images for training and another $100$ for evaluation.
The surface of a skull shape is constituted by the non-zero voxels (i.e., the  `1's), and we create defective skulls by removing a portion of such voxels from each image.
For the evaluation set, two defects are created for each image - one is similar to the defects in the training set while the other is significantly different in terms of its shape and size, as well as its position on the skull surface.
The dimension of the skull images is $64^3$.

\subsubsection{Network Architecture and Experimental Setup}
The neural network ($\sim 1$M trainable parameters) follows a standard auto-encoder architecture, in which five two-strided convolutional and deconvolutional layers are used for downsampling and upsampling respectively.
The output of the last convolutional layer is flattened and linearly mapped to an eight-dimensional latent code, which is then decoded by another linear layer before being passed on to the first deconvolution. 
The network takes as input a defective skull and learns to reconstruct its defectless counterpart. 

As a baseline we train the network using a Dice loss \citep{milletari2016v}, and a Dice loss combined with the FAM regularizer 
, which is applied to the second linear layer (of size $64\times 8$) of the network. 
We experimented with different coefficients $\lambda$ that weigh the regularizer against the Dice loss. All experiments use the Adam optimizer with a constant learning rate of $10^{-4}$.
The trained models are evaluated on the two aforementioned evaluation sets, using Dice similarity coefficient (DSC), Hausdorff distance (HD), and $95$ percentile Hausdorff distance (HD95). DSC is the main metric in practice for skull shape reconstruction~\citep{li2021autoimplant}, measuring how well two shapes overlap (the higher the better\footnote{The Dice loss (Figure~\ref{training_dice_loss}), on the contrary, is usually implemented as $1 - DSC$, which we minimize during training.}), while the distance measures i.e., HD and HD95 are supplementary.

\subsubsection{Results and Discussion}
Figure~\ref{training_dice_loss} shows the Dice loss curves under different weighting coefficients $\lambda$. 
Table~\ref{table:skullrec_score} shows the quantitative results on the two evaluation sets, and Figure~\ref{fig:training_dice_loss:boxplot} shows the distribution of the evaluation results for $\lambda=0.02, 0.002, 0.0006$ and the baseline.
The $DSC$ (100), $HD$ (100)  and $HD95$ (100) columns in Table \ref{table:skullrec_score} show the evaluation results at an intermediate training checkpoint (epoch $100$).

These results reveal several interesting findings: (i) At both the intermediate (epoch=100) and end checkpoint (epoch=200), the training loss of the baseline network is clearly lower than that of the regularized networks (Figure~\ref{training_dice_loss}), whereas its test accuracy is obviously worse than its regularized counterparts in terms of all metrics (Table \ref{table:skullrec_score}); (ii) The baseline network achieves higher test accuracy (DSC) at the intermediate checkpoint than at the end checkpoint, which is a clear indicator of overfitting, while the test accuracy of a properly regularized network (e.g., $\lambda=0.02, 0.002$) on either evaluation set 1 or evaluation set 2 keeps improving as training progresses; (iii) Even a very loose regularization (e.g., $\lambda=0.0006$) can prevent the Dice loss from decreasing until overfitting, as opposed to the baseline network (Figure \ref{training_dice_loss}); (iv) It is also worth mentioning that the scores on both evaluation sets stay essentially unchanged for the FAM-regularized network (e.g., $\lambda=0.02$), indicating that moderately altering the defects (e.g., defect shape, size, position) does not affect the network's performance, while in contrast, the baseline network performs worse on evaluation set $2$ than on evaluation set $1$ in terms of all metrics.

\begin{figure}
\centering
\includegraphics[width=0.9\linewidth]{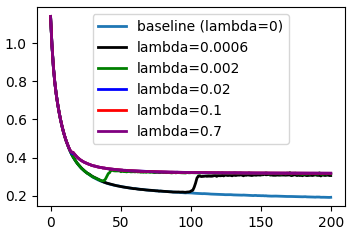}
 \caption{Curves of the Dice loss ($y$ axis) with respect to training epochs ($x$ axis), under different $\lambda$. Note that the red ($\lambda=0.1$) and purple ($\lambda=0.7$) lines overlap in this plot.}
\label{training_dice_loss}
\end{figure}

\begin{table*}[h!]
\centering
\caption{Quantitative Results for Skull Shape Reconstruction Given Different $\lambda$}
\scalebox{0.7}{
\begin{tabular}[t]{lccccccccccccc} 
\toprule
 \multirow{2}{*}{methods}  & \multicolumn{6}{c}{\textbf{evaluation set 1}} & \multicolumn{6}{c}{\textbf{evaluation set 2}}  \\ 

                       & $DSC$          & $DSC$ (100)     & $HD$           & $HD$ (100)     & $HD95$         & $HD95$ (100)  && $DSC$          & $DSC$ (100)    & $HD$           & $HD$ (100)     & $HD95$         & $HD95$ (100)   \\ 
\hline
baseline               & $0.6464$             & $0.6569$       & $7.0130$       & $7.1787$        & $2.0635$       & $2.0422$       && $0.6413$      & $0.6489$      & $7.1421$       & $7.1939$       & $2.0924$       & $2.1371$       \\ \hdashline
FAM,  $\lambda=0.0006$ & $0.7155$             & $0.6817$       & $6.5531$       & $6.7772$        & $1.8202$       & $\mbf{1.8281}$ && $0.7156$       & $0.6762$       & $6.5542$       & $7.0115$       & $1.8178$       & $1.9088$       \\
FAM,  $\lambda=0.002$  & $0.7173$             & $\mbf{0.7175}$ & $\mbf{6.4813}$ & $6.5478$        & $\mbf{1.8175}$ & $\mbf{1.8281}$ && $0.7175$       & $\mbf{0.7176}$   & $\mbf{6.4813}$ & $6.5478$       & $\mbf{1.8148}$ & $\mbf{1.8281}$ \\
FAM,  $\lambda=0.02$   & $\mbf{0.7176}$       & $0.7168$       & $6.5221$       & $\mbf{6.5271}$  & $1.8210$       & $1.8344$       && $\mbf{0.7176}$  & $0.7168$       & $6.5221$       & $\mbf{6.5271}$ & $1.8210$       & $1.8344$       \\
FAM,  $\lambda=0.1$    & $\mbf{0.7176}$       & $0.7169$       & $6.5085$       & $\mbf{6.5222}$  & $1.8210$       & $1.8345$       && $\mbf{0.7176}$  & $0.7169$         & $6.5085$       & $\mbf{6.5222}$         & $1.8210$       & $1.8345$              \\
FAM,  $\lambda=0.7$    & $\mbf{0.7177}$       & $0.7169$       & $6.5202$       & $6.5389$        & $1.8210$       & $1.8359$       && $\mbf{0.7177}$   & $0.7169$         & $6.5202$       & $6.5389$         & $1.8210$       & $1.8359$              \\
\bottomrule
\end{tabular}}
\label{table:skullrec_score}
\end{table*}

Choosing a proper $\lambda$ is important for a desired reconstructive performance. A large $\lambda$ enforces a flat(ter) curve of the loss with respect to the weights of the second linear layer, which is responsible for decoding the latent codes. However, over-regularization (in our case $\lambda=0.1, 0.7$) can lead to unvaried shape reconstructions by the decoder, since, in order for the loss to remain unchanged, the second linear layer has to give the same decoding for different latent codes~\footnote{Different skull shapes are expected to be encoded differently through the downsampling path of the auto-encoder.}. Therefore, the quantitative results for $\lambda=0.1, 0.7$ in Table \ref{table:skullrec_score} should be interpreted with care, i.e., the over-regularized networks `find' a universal reconstruction that somehow matches well with different evaluation cases (hence achieving relatively high DSC), which nevertheless defies the rule of case-specific reconstruction. 

\begin{figure}
\centering
\includegraphics[width=1\linewidth]{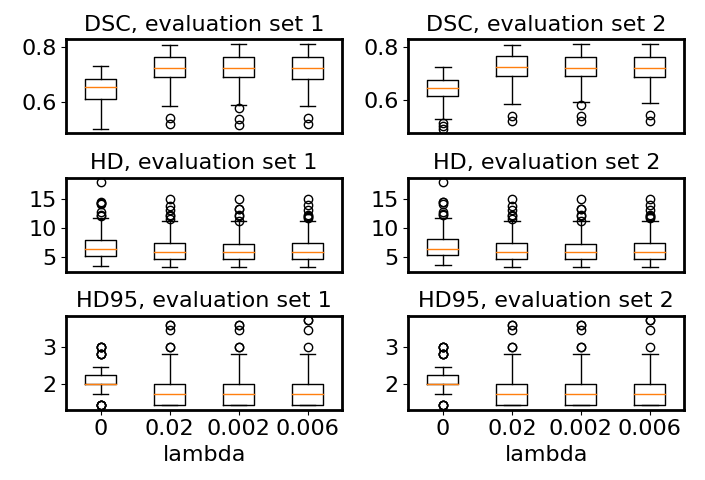}
 \caption{Boxplots of DSC, HD and HD95 given different $\lambda$ (x axis) on the two evaluation sets.}
\label{fig:training_dice_loss:boxplot}
\end{figure}

\subsection{Transformers}

Since the introduction of transformers \citep{attention_is_all_you_need}, large language models have revolutionized natural language processing by consistently pushing the state-of-the-art in various benchmark tasks \citep{devlin-etal-2019-bert, electra, deberta}. However, a recurring challenge in the fine-tuning process of these models is the occurrence of instabilities \citep{hua2021noise, mosbach2021on}. These instabilities can negatively impact the performance and reliability of the fine-tuned models. In the following section we demonstrate how the application of FAM can improve the downstream performance of transformers. We do not compare with SAM both for the reasons stated in the previous set of the experiments and also due to impossibility to integrate additional gradient iteration into the used implementation of BERT.

We fine-tune $BERT_{BASE}$ (110 million parameters) \citep{devlin-etal-2019-bert} to the Recognizing Textual Entailment (RTE) dataset \citep{rte_dataset} from the General Language Understanding Evaluation benchmark \citep{wang-etal-2018-glue}. The dataset consists of sentence pairs with binary labels that indicate whether the meaning of one sentence is entailed from its counterpart. In the past, this dataset was found to be particularly prone to instabilities \citep{phang2018sentence}.

\begin{table}[ht]
\centering
\caption{Results for the fine-tuning on the RTE validation set.}
\begin{tabular}[t]{lcc} 
\toprule
 &  Baseline & FAM \\
\hline
 \textbf{Accuracy}      & $0.67364$ & $\mathbf{0.6982}$  \\
\textbf{Standard Deviation} & $0.018$ & $\mathbf{0.0154}$  \\
 \textbf{Max}         & $0.6931$ & $\mathbf{0.7184}$ \\
\bottomrule
\end{tabular}
\label{table:rte-results}
\end{table}
In stark contrast to other experiments, we chose a much larger weighting coefficient $\lambda=3\mathrm{e}{6}$, as lower values had no influence on the training. Our training setup involved a learning rate of $\lambda=2\mathrm{e}{-5}$, a batch size of $32$, and a maximum sequence length of $128$ for $20$ epochs. We report the average development set accuracy across five runs with different random seeds. Table \ref{table:rte-results} presents the results of this experiment. 
We observed a progressive increase in validation loss throughout the training when the regularizer was not employed, indicating severe overfitting. While this phenomenon persisted with FAM, its effect was less pronounced, as depicted in Figure~\ref{training_rte_loss}.

\begin{figure}
\centering
\includegraphics[width=0.9\linewidth]{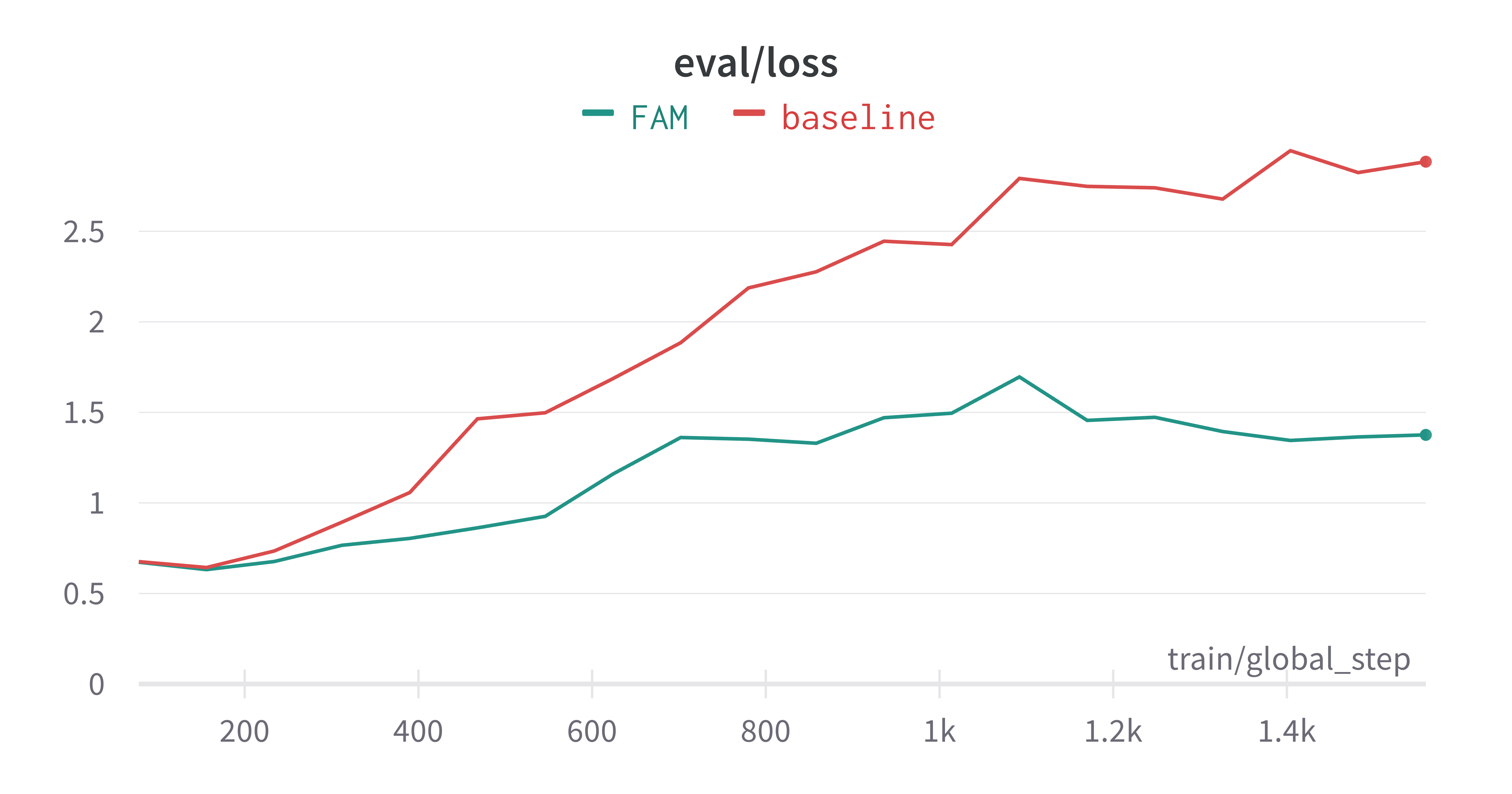}
 \caption{Development loss of the RTE training.}
\label{training_rte_loss}
\end{figure}

	\section{Discussion and Conclusion}
\label{sec:discussion}

We have shown that regularization based on the theoretically sound relative flatness measure improves generalization in a wide range of applications and model architectures, outperforming standard training and sometimes SAM. 

In our experiments (except for the skull reconstruction experiments, due to the specific architecture of the network), we have chosen the penultimate layer to compute relative flatness, as suggested by~\citet{petzka2021relative}. Their theory ensures that achieving flatness in any one layer suffices to reach good generalization.
We leave a comprehensive empirical study of the impact of the choice of layer (or even using multiple layers) on model quality for future work.
It can be also investigated whether flatness regularized on one of the layers also changes the flatness of other layers or not.

Relative flatness is connected to generalization under the assumption of locally constant labels in the representation.
This assumption holds already for the input space in many applications (e.g., image classification, and NLP)---the definition of adversarial examples hinges on this assumption. 
It implies, however, that flatness is not connected to generalization for tasks where the assumption is violated.
The recent study by~\citet{kaddourflat} supports this empirically by showing that regularizing wrt. flatness is not always beneficial. 
For future work it would be interesting to verify this study with FAM, testing the assumption of locally constant labels, and expanding it to further tasks.

While current implementation of the FAM regularizer allows for achieving better performance, the performance with respect to the space consumption can be improved, as well as computational time.
This currently also limits the applicability to convolutional layers, since treating them like a standard layer would increase the number of parameters greatly.
This can be overcome by determining the correct structure of the FAM regularizer for convolutional layers and is an interesting direction for future work. 


    
	\pagebreak
	{\small
    \bibliographystyle{icml2023}
    \bibliography{references}
    }
 	\vfill
 	\pagebreak
	\appendix
    \onecolumn
	\section{Hessian Computation and Approximation}
\label{sec:hesscomp}
In practice, the training time for FAM regularization depends on the method used for calculating the Hessian, respectively approximating its trace in case of the simplified relative flatness measure. In the following, we discuss several practical approaches in pytorch~\citep{paszke2019pytorch}.

\subsection{Computation of the Full Hessian}
Computing the Hessian, i.e., the second derivatives wrt. a neural network's weights, can straight-forwardly be done in pytorch using its autograd library. This method, however, is not optimized for runtime. The torch.autograd library also provides an experimental vectorized version of the Hessian computation. It uses a vectorization map as the backend to vectorize calls to autograd.grad, which means that it only invokes it once instead of once per row, making it more computationally efficient. We compare the \emph{non-vectorized} to the \emph{vectorized} variant of torch.autograd. Recently, the pytorch library functorch (in beta) provided a fast Hessian computation method build on top of the autograd library and also using a vectorization map. Additionally, it uses XLA, an optimized compiler for machine learning that accelerates linear algebra computations. This further accelerates Hessian computation, but does not yet work with all neural networks---in particular, the functorch Hessian computation requires batch normalization layers to not track the running statistics of training data. In Figure~\ref{fig:Hessian_runtime_comp} we show that using the vectorized approach substantially reduces computation time by up to three orders of magnitude. For larger Hessians, the functorch library further improves runtime over the vectorized autograd method by an order of magnitude. 
All experiments are performed on an NVIDIA RTX A6000 GPU.

\begin{figure}
\centering
\includegraphics[width=0.5\linewidth]{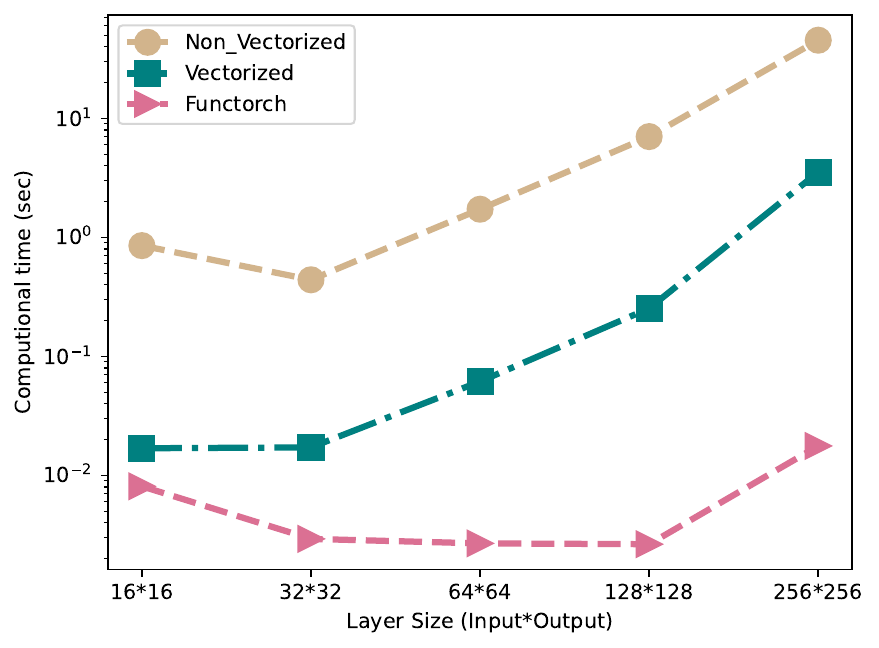}
 \caption{Comparing non-vectorized and vectorized autograd, as well as functorch in terms of the computation time for computing the full Hessian of a single neural network layer for different layer sizes.}
\label{fig:Hessian_runtime_comp}
\end{figure}

\subsection{Computation of the Trace of the Hessian}
When the layers are high-dimensional, forming the full Hessian can be memory and computationally expensive. Since FAM requires the calculation of the trace of a Hessian, we apply the trick of using the Hutchinson's method~\citep{hutchinson1990stochastic} to approximate the trace of the Hessian. The version of Hutchinson's trick we use is described as follows:

 Let $A \in \mathbb{R}^{D \times D}$ and $v \in \mathbb{R}^D$ be a random vector such that $\mathbb{E}\left[v v^T\right]=I$. Then,
$$
\operatorname{Tr}(A)=\mathbb{E}\left[v^T A v\right]=\frac{1}{V} \sum_{i=1}^V v_i^T A v_i .
$$
where $v$ is generated using Rademacher distribution and $V$ is the number of Monte Carlo samples. The intuition behind this method is that by averaging over many random vectors, we can obtain an estimate of the trace of the matrix. It has been proved that the trace estimator converges with the smallest variance to the trace if we use Rademacher random numbers~\citep{tropp2012user}. This method is in general very useful when we need to compute the trace of a function of a matrix.
\begin{figure}
\centering
\includegraphics[width=0.5\linewidth]{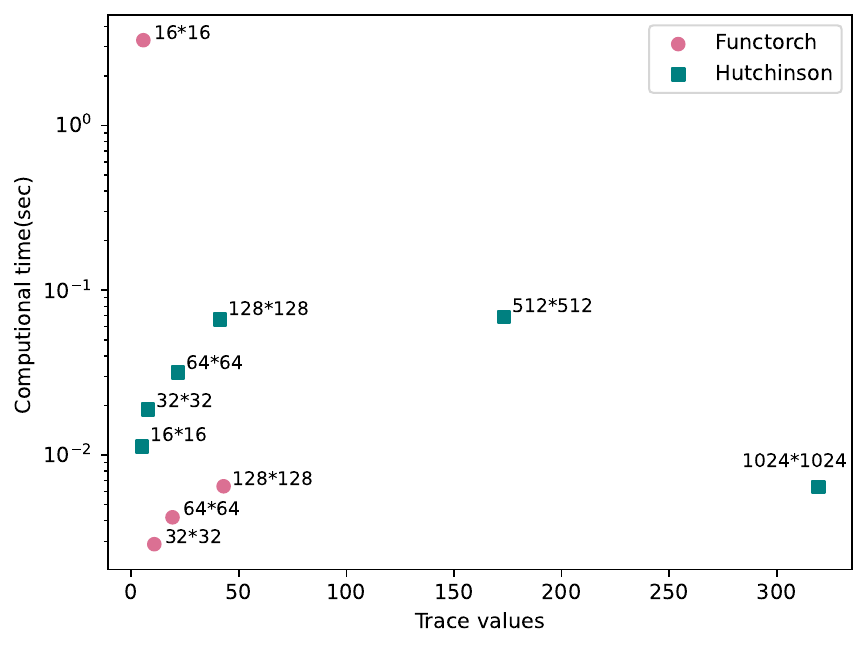}
 \caption{Computational time of the trace of the hessian for different layer sizes using Functorch and Hutchinson's method}
\label{Hessian_2}
\end{figure}

Computational time for the direct functorch computation of the Hessian trace and for the Hutchinson's trick is shown in Figure \ref{Hessian_2}.

\end{document}